\documentclass[letterpaper]{article} 
\usepackage{aaai23}  
\usepackage{times}  
\usepackage{helvet}  
\usepackage{courier}  
\usepackage[hyphens]{url}  
\usepackage{graphicx} 
\def\UrlFont{\rm}  
\usepackage{natbib}  
\usepackage{caption} 
\usepackage{algorithm}
\usepackage{algorithmic}
\usepackage{url}            
\usepackage{booktabs}       
\usepackage{amsfonts}       
\usepackage{nicefrac}       
\usepackage{microtype}      
\usepackage{xcolor}         
\usepackage{graphicx} 
\usepackage{amsmath}
\usepackage{amsthm}
\usepackage{subfig}
\usepackage{multirow,diagbox,enumitem}

\newtheorem{theorem}{Theorem}

\newcommand{\modified}[1]{{#1}}
\newcommand{\noteblue}[1]{#1}
\newcommand{\nb}[1]{#1}

\newcommand{\vpara}[1]{\vspace{0.05in}\noindent\textbf{#1 }}
\newcommand{\vpait}[1]{\vspace{0.05in}\noindent\textit{#1 }}

\newtheorem{lemma}{Lemma}
\newtheorem{definition}{Definition}

\usepackage{newfloat}
\usepackage{listings}
\DeclareCaptionStyle{ruled}{labelfont=normalfont,labelsep=colon,strut=off} 
\floatstyle{ruled}
\newfloat{listing}{tb}{lst}{}
\floatname{listing}{Listing}
\title{DropMessage: Unifying Random Dropping for Graph Neural Networks}
\author{
    Taoran Fang\textsuperscript{\rm 1}, 
    Zhiqing Xiao\textsuperscript{\rm 1},
    Chunping Wang\textsuperscript{\rm 2},
    Jiarong Xu\textsuperscript{\rm 3},
    Xuan Yang\textsuperscript{\rm 1},
    Yang Yang\textsuperscript{\rm 1}\thanks{Corresponding author.} \\
}
\affiliations{
    \textsuperscript{\rm 1}Zhejiang University, 
    \textsuperscript{\rm 2}FinVolution Group, 
    \textsuperscript{\rm 3}Fudan University \\

    \{fangtr, yangya\}@zju.edu.cn, wangchunping02@xinye.com,
    jiarongxu@fudan.edu.cn
}

\usepackage{bibentry}

\begin{document}

\maketitle

\begin{abstract}
Graph Neural Networks (GNNs) are powerful tools for graph representation learning. 
Despite their rapid development, GNNs also face some challenges, such as \textit{over-fitting}, \textit{over-smoothing}, and \textit{non-robustness}.
Previous works indicate that these problems can be alleviated by \textit{random dropping} methods, which integrate \nb{augmented data} into models by randomly masking parts of the input. 
However, some open problems of random dropping on GNNs remain to be solved. 
First, it is challenging to find a universal method that are suitable for all cases considering the divergence of different datasets and models.
Second, \nb{augmented data} introduced to GNNs causes the incomplete coverage of parameters and unstable training process.
\nb{Third, there is no theoretical analysis on the effectiveness of random dropping methods on GNNs.}
\modified{In this paper, we propose a novel random dropping method called \textit{DropMessage}, which performs dropping operations directly on the propagated messages during the message-passing process.}
\nb{More importantly, we find that DropMessage provides a unified framework for most existing random dropping methods, based on which we give theoretical analysis of their effectiveness.} 
Furthermore, we elaborate the superiority of DropMessage: it stabilizes the training process by reducing sample variance; it keeps information diversity from the perspective of information theory, enabling it become a theoretical upper bound of other methods.
To evaluate our proposed method, we conduct experiments that aims for multiple tasks on five public datasets and two industrial datasets with various backbone models. 
The experimental results show that DropMessage has the advantages of both effectiveness and generalization, \nb{and can significantly alleviate the problems mentioned above}.
Our code is available at: https://github.com/zjunet/DropMessage.
\end{abstract}

\section{Introduction}
Graphs, ubiquitous in the real world, are used to present complex relationships among various objects in numerous domains such as social media (social networks), finance (trading networks), and biology (biological networks). 
As powerful tools for representation learning on graphs, graph neural networks (GNNs) have attracted considerable attention recently~\cite{defferrard2016convolutional, kipf2016semi, velivckovic2017graph, ding2018semi}.  
In particular, GNNs adopt a message-passing schema \cite{gilmer2017neural}, in which each node aggregates information from its neighbors in each convolutional layer, and have been widely applied in various downstream tasks such as node classification \cite{kipf2016semi}, link prediction \cite{kipf2016variational}, vertex clustering \cite{ramaswamy2005distributed}, and recommendation systems \cite{ying2018graph}.

Yet, despite their rapid development, training GNNs on large-scale graphs is facing several challenges such as \textit{over-fitting}, \textit{over-smoothing}, and \textit{non-robustness}. 
Indeed, compared to other data forms, gathering labels for graph data is expensive and inherently biased, which limits the generalization ability of GNNs due to over-fitting. 
Besides, representations of different nodes in a GNN tend to become indistinguishable as a result of aggregating information from neighbors recursively.
This phenomenon of over-smoothing prevents GNNs from effectively modeling the higher-order dependencies from multi-hop neighbors~\cite{Li2018DeeperII, Xu2018RepresentationLO,Chen2020MeasuringAR,Zhao2020PairNormTO,Oono2020GraphNN,Oono2019OnAB}. 
Recursively aggregating schema makes GNNs vulnerable to the quality of input graphs~\cite{zhu2019robust,Zgner2018AdversarialAO}. 
In other words, noisy graphs or adversarial attacks can easily influence a GNN's performance. 

The aforementioned problems can be helped by \textit{random dropping} methods~\cite{Hinton2012ImprovingNN,rong2019dropedge,feng2020graph},
which integrate \nb{augmented data} into models by randomly masking parts of the input. 
These methods \cite{Maaten2013LearningWM, Matsuoka1992NoiseII, Bishop1995TrainingWN, Cohen2019CertifiedAR} focus on randomly dropping or sampling existing information, and can also be considered as a data augmentation technique. 
Benefiting from the advantages of being unbiased, adaptive, and free of parameters, random dropping methods have greatly contributed to improving the performance of most GNNs. 

\begin{figure*}
\centering
\includegraphics[width=0.9\textwidth]{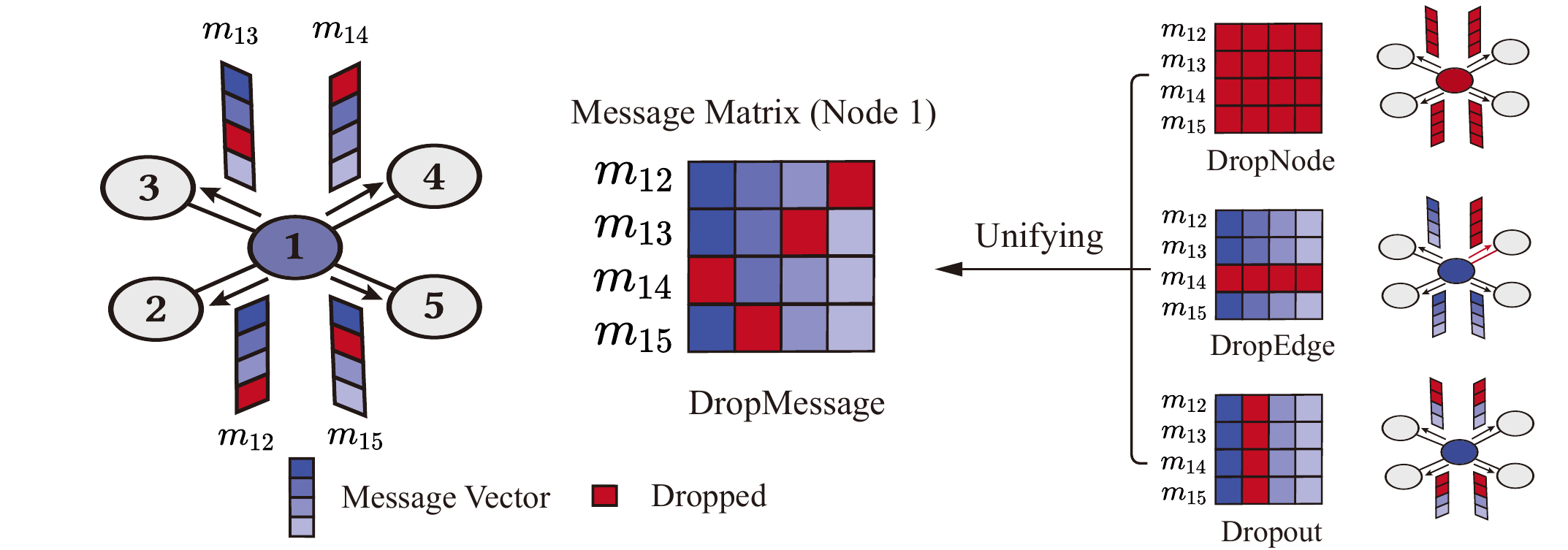}
\caption{
\nb{Illustrations of DropMessage and other existing random dropping methods.
Considering the messages propagated by the center node (\emph{i.e.}, Node 1), DropMessage allows to propagate distinct messages to different neighbor nodes, and its induced message matrix can be arbitrary.
The induced message matrices of other methods obey some explicit constraints and can be regarded as special forms of DropMessage.}
}
\label{fig:model}
\end{figure*}

However, some open questions related to random dropping methods on GNNs still exist.
First, a general and critical issue of existing random dropping methods is that \nb{augmented data} introduced to GNNs make parameters difficult to converge and the training process unstable. 
Moreover, it is challenging to find an optimal dropping method suitable to all graphs and models,
because different graphs and models are equipped with their own properties and the model performance can be influenced greatly by employing various dropping strategies.
Furthermore, the answer to how to choose a proper dropping rate when applying these methods is still unclear, and so far no theoretical guarantee has been provided to explain why random dropping methods can improve the performance of a GNN. 

In this paper, we propose a novel random dropping method called DropMessage, which can be applied to all message-passing GNNs.
As Figure \ref{fig:model} suggests, existing random dropping methods perform dropping on either the node feature matrix \cite{Hinton2012ImprovingNN,feng2020graph} or the adjacency matrix \cite{rong2019dropedge}, while our DropMessage performs dropping operations on the propagated messages, which allows the same node to propagate different messages to its different neighbors.
Besides, we unify existing dropping methods into our framework and demonstrate theoretically that conducting random dropping methods on GNNs is equivalent to introducing additional regularization terms to their loss functions, which makes the models more robust.
Furthermore, we also elaborate the superiority of our DropMessage whose sample variance is much smaller and training process is more stable.
From the perspective of information theory, DropMessage keeps the property of information diversity, and is theoretically regarded as an upper bound of other random dropping methods.
To sum up, the contributions of this paper are as follows:
\begin{itemize}
    \item{
    We propose a novel random dropping method, called DropMessage, for all message-passing GNNs. 
    Existing random dropping methods on GNNs can be unified into our framework via performing masking in accordance with the certain rule on the message matrix.
    \nb{In other words, these methods can be regarded as one special form of DropMessage.}
    }
    \item{
    \nb{We theoretically demonstrate the effectiveness of the random dropping methods, filling the gap in this field.}
    }
    \item{
    We conduct sufficient experiments for different downstream tasks, and the experimental results show that DropMessage can better alleviate over-fitting, over-smoothing, and non-robustness compared to other existing random dropping methods.
    }

\end{itemize}

\section{Related Work}
To sort out the key logic of our work, we first review some related work about random dropping methods with a particular focus on GNNs. 


In general, random dropping can be regarded as a form of feature-noising schema that alleviates over-fitting by artificially corrupting the training data. 
As a representative work, Dropout is first introduced by Hinton et al. \cite{Hinton2012ImprovingNN} and has been proved to be effective in many scenarios~\cite{abu1990learning,Burges1996ImprovingTA,simard1998transformation,Rifai2011TheMT,Maaten2013LearningWM}.
Besides, Bishop \cite{Bishop1995TrainingWN} demonstrates the equivalence of corrupted features and $\textit{L}_{2}$-type regularization. 
Wager et al. \cite{Wager2013DropoutTA} show that the dropout regularizer is first-order equivalent to an $\textit{L}_{2}$ regularizer that being applied after scaling the features by an estimate of the inverse diagonal Fisher information matrix.

With the rapid development of GNNs, random dropping has also been generalized to the graph field, thus leading to three most common methods: \textit{Dropout}~\cite{Hinton2012ImprovingNN}, \textit{DropEdge}~\cite{rong2019dropedge} and \textit{DropNode}~\cite{feng2020graph}. 
Dropout performs random dropping operation on the node feature matrix, while DropEdge and DropNode, as the name implies, respectively act on the adjacency matrix (edges) and nodes. 
These random dropping methods can also be regarded as special forms of data augmentation \cite{Shorten2019ASO, FridAdar2018GANbasedSM, Buslaev2020AlbumentationsFA, Ding2022DataAF, Velickovic2019DeepGI}, with the advantage of not requiring parameter estimation \cite{Papp2021DropGNNRD, Luo2021LearningTD, Chen2018FastGCNFL, Zeng2020GraphSAINTGS} and easy to apply. 
All the methods mentioned above can be used to alleviate \textit{over-fitting} and \textit{over-smoothing} on GNNs. 
However, they can achieve effective performance only on some specific datasets and GNNs. 
The question of how to find an optimal dropping method that suitable for most cases still remains to be explored.
\noteblue{Moreover, there is no theoretical explanation about the effectiveness of random dropping methods on GNNs, which adds some ambiguity to the function of these methods.}


\section{Notations and Preliminaries}

\vpara{Notations.} 
Let $\mathbf{G}=\left(\mathbf{V}, \mathbf{E}\right)$ represent the graph, 
where $\mathbf{V}=\left\{v_{1}, \ldots, v_{n}\right\}$ denotes the set of $n$ nodes, 
and $\mathbf{E} \subseteq \mathbf{V} \times \mathbf{V}$ is the set of edges between nodes. 
The node features can be denoted as a matrix 
$\mathbf{X} = \left\{ x_{1}, \ldots, x_{n} \right\} \in \mathbb{R}^{n \times c}$, 
where $x_{i}$ is the feature vector of the node $v_i$, and $c$ is the dimensionality of node features.
The edges describe the relations between nodes and can be represented as an adjacent matrix $\mathbf{A} = \left\{ a_{1}, \ldots, a_{n} \right\} \in \mathbb{R}^{n \times n}$, 
where $a_{i}$ denotes the $i$-th row of the adjacency matrix, and $\mathbf{A} \left(i, j\right)$ denotes the relation between nodes $v_{i}$ and $v_{j}$.
Also, the node degrees are given by $ \mathbf{d}=\left\{ d_{1}, \ldots, d_{n} \right\}$, where $d_{i}$ computes the sum of edge weights connected to node $v_i$.
Meanwhile, the degree of the whole graph is calculated by $\mathbf{d(G)}=\sum_{i}^{n}d_{i}$.
When we apply message-passing GNNs on $\mathbf{G}$, 
the message matrix can be represented as $\mathbf{M}=\left\{m_{1}, \ldots, m_{k}\right\} \in \mathbb{R}^{k \times c'}$, 
where $m_{i}$ is a message propagated between nodes, $k$ is the total number of messages propagated on the graph, \modified{and $c'$ is the dimension number of the messages}.

\vpara{Message-passing GNNs.} 
Most of the existing GNN models 
adopt the message-passing framework, where each node sends messages to its neighbors and simultaneously receives messages from its neighbors.
In the process of the propagation, node representations are updated based on node feature information and messages from neighbors, which can be formulated as 

\begin{equation}
h^{(l+1)}_{i}=\gamma^{(l)}(h^{(l)}_{i},\mathcal{AGG}_{j\in\mathcal{N}(i)}(\phi^{(l)}(h_{i}^{(l)},h_{j}^{(l)},e_{j,i})))
\end{equation}

\noindent where $h^{(l)}_{i}$ denotes the hidden representation of node $v_i$ in the $l$-th layer, and $\mathcal{N}(i)$ is a set of nodes adjacent to node $v_i$; 
$e_{j,i}$ represents the edge from node $j$ to node $i$; 
$\phi^{(l)}$ and $\gamma^{(l)}$ are differentiable functions;
and $\mathcal{AGG}$ represents the aggregation operation. 
\modified{From the perspective of the message-passing schema, we can gather all the propagated messages into a message matrix $\mathbf{M} \in \mathbb{R}^{k \times c'}$.
Specifically, each row of the message matrix $\mathbf{M}$ corresponds to a message propagated on a directed edge, which can be expressed as below:
\begin{align}
\mathbf{M}^{(l)}_{(i,j)}=\phi^{(l)}(h_{i}^{(l)},h_{j}^{(l)},e_{j,i}^{(l)}) \nonumber
\end{align}
where $\phi$ denotes the mapping that generates the messages, $c'$ is the dimension number of the messages, and the row number $k$ of the message matrix $\mathbf{M}$ is equal to the directed edge number in the graph.}


\section{Our Approach}
In this section, we introduce our proposed \textit{DropMessage}, which can be applied to all message-passing GNNs. 
We first describe the details of our approach, and further prove that the most common existing random dropping methods, i.e., Dropout, DropEdge and DropNode, can be unified into our framework.  
Based on that, we give a theoretical explanation of the effectiveness of these methods. 
After that, we theoretically analyze the superiority of DropMessage in terms of stabilizing the training process and keeping information diversity. 
Finally, we derive a theoretical upper bound to guide the selection of dropping rate $\delta$. 

\subsection{DropMessage}
\label{section41} 
\vpara{Algorithm description.}
Different from existing random dropping methods, DropMessage performs 
directly on the message matrix $\mathbf{M}$ instead of the feature matrix or the adjacency matrix. 
More specifically, DropMessage conducts dropping on the message matrix with the dropping rate $\delta$, 
which means that $\delta|\mathbf{M}|$ elements of the message matrix will be masked in expectation. 
Formally, this operation can be regarded as a sampling process. 
For each element $\mathbf{M}_{i,j}$ in the message matrix, we generate an independent mask $\epsilon_{i,j}$ to determine whether it will be preserved or not, according to a Bernoulli distribution $\epsilon_{i,j}\sim Bernoulli(1-\delta)$. 
Then, we obtain the perturbed message matrix $\mathbf{\widetilde{M}}$ by multiplying each element with its mask. 
Finally, we scale $\mathbf{\widetilde{M}}$ with the factor of $\frac{1}{1-\delta}$ to guarantee that the perturbed message matrix is equal to the original message matrix in expectation. 
Thus, the whole process can be expressed as $\mathbf{\widetilde{M}}_{i,j}=\frac{1}{1-\delta}\epsilon_{i,j}\mathbf{M}_{i,j}$, where $\epsilon_{i,j}\sim Bernoulli(1-\delta)$. 
The applied GNN model then propagates information via the perturbed message matrix  $\mathbf{\widetilde{M}}$ instead of the original message matrix. 
It should be moted that DropMessage only affects on the training process. 

\nb{
\vpara{Practical implementation.}
In practice, we do NOT need to generate the complete message matrix explicitly, because each row in the message matrix represents a distinct directed edge in the graph, and our proposed DropMessage can be applied to every directed edge independently.
This property allows DropMessage to be easily parallelized, \emph{e.g.}, edge-wise, node-wise, or batch-wise, and to be applied to the message-passing backbone model \textit{without increasing time or space complexity}.
}

\vpara{Unifying random dropping methods.} 
As we have mentioned above, DropMessage differs from existing methods by directly performing on messages instead of graphs. 
However, in intuition, the dropping of features, edges, nodes or messages will all eventually act on the message matrix. 
It inspires us to explore the theoretical connection between different dropping methods. 
As a start, we demonstrate that Dropout, DropEdge, DropNode, and DropMessage can all be formulated as Bernoulli sampling processes in Table \ref{tab:methods}. 
More importantly, we find that existing random dropping methods are actually special cases of DropMessage, and thus can be expressed in a uniform framework. 


\begin{table}[t]
    \centering
    \caption{Overview of different random dropping methods in a view of Bernoulli sampling process.}
    \label{tab:methods}
    \begin{tabular}{cc}
    \toprule
    Method & Formula \\
    \midrule
    Dropout & $\mathbf{\widetilde{X}}_{i,j}=\epsilon \mathbf{X}_{i,j}$ \\
    DropEdge & $\mathbf{\widetilde{A}}_{i,j}=\epsilon \mathbf{A}_{i,j}$ \\
    DropNode & $\mathbf{\widetilde{X}}_{i}=\epsilon \mathbf{X}_{i}$ \\
    DropMessage & $\mathbf{\widetilde{M}}_{i,j}=\epsilon \mathbf{M}_{i,j}$ \\
    \bottomrule
    \footnotesize
    $s.t.\ \epsilon\sim Bernoulli(1-\delta)$
    \end{tabular}
\end{table}




\begin{lemma}
\label{Lamma1}
Dropout, DropEdge, DropNode, and DropMessage perform random masking on the message matrices in accordance with certain rules. 
\end{lemma}

We provide the equivalent operation on the message matrix of each method below. 

\vpait{Dropout.} 
Dropping the elements 
$X_{drop}=\{\mathbf{X}_{i,j}|\epsilon_{i,j}=0\}$ 
in the feature matrix $\mathbf{X}$ is equivalent to masking elements 
$M_{drop}=\{\mathbf{M}_{i,j}|source(\mathbf{M}_{i,j})\in X_{drop}\}$ 
in the message matrix $\mathbf{M}$, where $source(\mathbf{M}_{i,j})$ indicates which element in the feature matrix that $\mathbf{M}_{i,j}$ corresponds to.

\vpait{DropEdge.}
Dropping the elements 
$E_{drop}=\{\mathbf{E}_{i,j}|\mathbf{A}_{i,j}=1\; and \; \epsilon_{i,j}=0\}$ 
in the adjacency matrix $\mathbf{A}$ is equivalent to masking elements
$M_{drop}=\{\mathbf{M}_{i}|edge(\mathbf{M}_{i})\in E_{drop}\}$ 
in the message matrix $\mathbf{M}$, where $edge(\mathbf{M}_{i})$ indicates which edge that $\mathbf{M}_{i}$ corresponds to.

\vpait{DropNode.} 
Dropping the elements 
$V_{drop}=\{\mathbf{X}_{i}|\epsilon_{i}=0\}$ 
in the feature matrix $\mathbf{X}$ is equivalent to masking elements 
$M_{drop}=\{\mathbf{M}_{i}|node(\mathbf{M}_{i})\in V_{drop}\}$ 
in the message matrix $\mathbf{M}$, where $node(\mathbf{M}_{i})$ indicates which row in the feature matrix that $\mathbf{M}_{i}$ corresponds to.

\vpait{DropMessage.} 
This method directly performs random masking on the message matrix $\mathbf{M}$. 

According to above descriptions, we find DropMessage conduct finest-grained masking on the message matrix, which makes it the most flexible dropping method, and other methods can be regarded as a special form of DropMessage. 

\vpara{Theoretical explanation of effectiveness.} 
Previous studies have explored and explained why random dropping works in the filed of computer vision~\cite{Wager2013DropoutTA, Wan2013RegularizationON}. 
However, to the best of our knowledge, the effectiveness of random dropping on GNNs has not been studied yet. 
To fill this gap, based on the unified framework of existing methods, 
we next provide a theoretical analysis. 


\begin{theorem}
\label{Regularization}
Unbiased random dropping on GNNs methods introduce an additional regularization term into the objective functions, which makes the models more robust. 
\end{theorem}

\begin{proof}
For analytical simplicity, we assume that the downstream task is a binary classification and we apply 
a single layer GCN~\cite{kipf2016semi} as the backbone model, which can be formulated as $\mathbf{H}=\mathbf{\overline{B}}\mathbf{M}\mathbf{W}$, 
where $\mathbf{M}$ denotes the message matrix, $\mathbf{W}$ denotes the transformation matrix, 
$\mathbf{B} \in \mathbb{R}^{n \times k}$ indicates which messages should be aggregated by each node and $\mathbf{\overline{B}}$ is its normalized form. 
Also, we adopt sigmoid as non-linear function and present the result as $\mathbf{Z}=sigmoid(\mathbf{H})$. 
When we use cross-entropy as loss function, the objective function can be expressed as follows:
\begin{align}
L_{CE}&=\sum_{j,y_{j}=1}log(1+e^{-h_{j}})+\sum_{k,y_{k}=0}log(1+e^{h_{k}})
\end{align}

When performing random dropping on graphs, we use the perturbed message matrix $\mathbf{\widetilde{M}}$ instead of the original message matrix $\mathbf{M}$. 
Thus, the objective function in expectation can be expressed as follows: 
\begin{align}
\label{eqLce}
E(\widetilde{L}_{CE})&=L_{CE}
+\sum_{i}\frac{1}{2}z_{i}(1-z_{i})Var(\tilde{h}_{i})
\end{align}

More details of the derivation can be found in Appendix.
As shown in Equation \ref{eqLce}, random dropping methods on graphs introduce an extra regularization to the objective function. 
For binary classification tasks, this regularization enforces the classification probability approach to 0 or 1, \noteblue{thus a clearer judgment can be obtained}.
\noteblue{By reducing the variance of $\tilde{h}_{i}$, random dropping methods motivate the model to extract more essential high-level representations.}
Therefore, the robustness of the models is enhanced.
\nb{It is noted that Equation \ref{eqLce} can be well generalized to multi-classification tasks by extending dimension of the model output.
Formally, when dealing with the multi-classification task, the final objective function can be expressed as $E(\widetilde{L}_{CE})=L_{CE}+\sum_{i}\frac{1}{2}z_{i}^{c_i}(1-z_{i}^{c_i})Var(\tilde{h}_{i}^{c_i})$, where $c_i$ is the label of node $v_i$, and the superscript indicates which dimension of the vector is selected.
}
\end{proof}

\subsection{Advantages of DropMessage}
\label{section42}
We give two additional analysis to demonstrate the advantages of DropMessage on two aspects: stabilizing the training process and keeping diverse information.  

\vpara{Reducing sample variance.}
All random dropping methods are challenged by the problem of unstable training process. 
As existing works suggest, it is caused by the random noises introduced into each training epoch. These noises then add the difficulty of parameter coverage and the unstability of training process.
\noteblue{Generally, \textit{sample variance} can be used to measure the degree of stability.
According to Table \ref{tab:methods}, the input of each training epoch can be regarded as a random sample of the whole graph, and the sample variance is calculated by the average difference of every two independent samples.
}
Compared with other random dropping methods, DropMessage effectively alleviates the aforementioned problem by reducing the sample variance.

\begin{theorem}
\label{t2}
DropMessage presents the smallest sample variance among existing random dropping methods on message-passing GNNs with the same dropping rate $\delta$.
\end{theorem}

We leave the proof 
in Appendix. 
Intuitively, DropMessage independently determines whether an element in the message matrix is masked or not, which is exactly the smallest Bernoulli trail for random dropping on the message matrix. 
By reducing the sample variance, DropMessage 
diminishes the difference of message matrices among distinct training epochs, which stabilizes the training process and expedites the convergence. 
\noteblue{
The reason why DropMessage has the minimum sample variance is that it is the finest-grained random dropping method for GNN models.
When applying DropMessage, each element $\mathbf{M}_{i,j}$ will be independently judged that whether it should be masked.}

\vpara{Keeping diverse information.} 
\label{section44}
In the following, we compare different random dropping methods with their degree of losing \textit{information diversity}, from the perspective of information theory. 

\begin{definition}
\label{def44}
\nb{
The information diversity consists of feature diversity and topology diversity. 
We define feature diversity as $\textit{FD}_{G}=card(\{\Vert M_{SN(v_i),l}\Vert_0\geq 1\})$, where $v_i\in \mathbf{V}$, $l\in[0,c)$, $SN(v_i)$ indicates the slice of the row numbers corresponding to the edges sourced from $v_i$; 
topology diversity is defined as $\textit{TD}_{G}=card(\{\Vert M_j\Vert_0\geq 1\})$, where $j\in [0,k)$. 
$M\in \mathbb{R}^{k \times c}$ represents the message matrix, $\Vert\cdot\Vert_0$ calculates the zero norm of the input vector, and $card(\cdot)$ counts the number of elements in the set.}
\end{definition}


\nb{In other words, feature diversity is defined as the total number of preserved feature dimensions from distinct source nodes; 
topology diversity is defined as the total number of directed edges propagating at least one dimension message.}  
With the above definition, we claim that a method possesses the ability of \textit{keeping information diversity} only under the condition where \noteblue{neither the feature diversity nor the topology diversity decreases after random dropping.} 

\begin{lemma}
\label{lamma45}
None of Dropout, DropEdge, and DropNode is able to keep information diversity. 
\end{lemma}

According to Definition \ref{def44}, when we drop an element of the feature matrix $\mathbf{X}$, all corresponding elements in the message matrix are masked and the feature diversity is decreased by $1$.
When we drop an edge in adjacency matrix, the corresponding two rows for undirected graphs in the message matrix are masked and the topology diversity is decreased by $2$. 
Similarly, when we drop a node, \emph{i.e.}, a row in the feature matrix, elements in the corresponding rows of the message matrix are all masked. 
Both the feature diversity and the topology diversity are therefore decreased. 
Thus, for all of these methods, their feature and topology information cannot be completely recovered by propagated messages, leading to the loss of information diversity.

\begin{theorem}
\label{th46}
DropMessage can keep information diversity in expectation when $\delta_{i} \leq 1-min(\frac{1}{d_{i}},\frac{1}{c})$, where $\delta_{i}$ is the dropping rate for node $v_i$, 
$d_{i}$ is the out-degree of $v_i$, and $c$ is the feature dimension.
\end{theorem}

\begin{proof}
DropMessage conducts random dropping directly on message matrix $\mathbf{M}$. To keep the diversity of the topology information, we expect that at least one element of each row in message matrix $\mathbf{M}$ can be preserved in expectation: 
\begin{equation}
\label{delta1}
E(|\mathbf{M}_{f}|)\geq 1\Rightarrow (1-\delta)c \geq 1 \Rightarrow \delta \leq 1-\frac{1}{c}
\end{equation}

To keep the diversity of the feature information, we expect that for every element in the feature matrix $\mathbf{X}$, at least one of its corresponding elements in the message matrix $\mathbf{M}$ is preserved in expectation: 
\begin{align}
\label{delta2}
E(|\mathbf{M}_{e}|)\geq 1\Rightarrow (1-\delta_{i})d_{i} \geq 1 \Rightarrow \delta_{i} \leq 1-\frac{1}{d_{i}}
\end{align}


Therefore, to keep the information diversity, the dropping rate $\delta_{i}$ should satisfy both Equation \ref{delta1} and Equation \ref{delta2} as 
\begin{equation}
\label{equation7}
\delta_{i} \leq 1-min(\frac{1}{d_{i}},\frac{1}{c})    
\end{equation}

\end{proof}
From the perspective of information theory, 
a random dropping method with the capability of keeping information diversity can preserve more information and theoretically perform better than those without such capability.
Thus, it can explain why our method performs better than those existing dropping methods.
Actually, we may only set one dropping rate $\delta$ for the whole graph rather than for each node in practice.
\noteblue{
Consequently, both DropMessage and other methods may lose some information.
However, DropMessage still preserves more information than other methods with the same dropping rate even under this circumstance.
It is demonstrated that DropMessage remains its advantage in real-world scenarios.
}

\section{Experiments}

\subsection{Experimental Setup}
We empirically validate the effectiveness and adaptability of our proposed DropMessage in this section. 
In particular, we explore the following questions: 
1) Does DropMessage outperform other random dropping methods on GNNs? 
2) Could DropMessage further improve the robustness and training efficiency of GNNs? 
3) Does information diversity (described in Definition \ref{def44}) matter in GNNs? 

\vpara{Datasets.} We employ 7 graph datasets in our experiments, including 5 public datasets \textit{Cora, CiteSeer, PubMed, ogbn-arxiv, Flickr} and 2 industrial datasets \textit{FinV, Telecom}. 

\begin{itemize}
    \item {\textit{Cora, CiteSeer, PubMed, ogbn-arxiv:}} 
    These 4 different citation networks are widely used as graph benchmarks~\cite{sen2008collective,hu2020ogb}. We conduct node classification tasks on each dataset to determine the research area of papers/researchers.
    We also consider link prediction on the first three graphs to predict whether one paper cites another.  
    \item {\textit{Flickr:}} It is provided by Flickr, the largest photo-sharing website~\cite{Zeng2020GraphSAINTGS}. One node in the graph represents one image uploaded to Flickr. If two images share some common properties (\emph{e.g.}, same geographic location, same gallery, or comments by the same user), an edge between the nodes of these two images will appear.
    We conduct the node classification task that aims to categorize these images into 7 classes determined by their tags. 
    \item {\textit{FinV, Telecom:}} These are two real-world mobile communication networks provided by FinVolution Group \cite{Yang2019UnderstandingDB} and China Telecom \cite{Yang2021MiningFA}, respectively. 
    In the two datasets, nodes represent users, and edges indicate the situation where two users have communicated with each other at a certain frequency. The task is to identify whether a user is a default borrower or a telecom fraudster.
\end{itemize}

\begin{table*}[htbp]
    \centering
    \caption{Comparison results of different random dropping methods. 
    The best results are in bold, while the second-best ones are underlined.
    }
    \label{tab:result_overall}
    \resizebox{1\textwidth}{!}{
    \begin{tabular}{c|cccc||ccc|||ccc}
    \toprule
    \multirow{2}{*}{\diagbox [width=11em, trim=lr] {Model}{Task \& Dataset}} & \multicolumn{7}{c|||}{Node classification} & \multicolumn{3}{c}{Link prediction}\\
    \cmidrule{2-8} \cmidrule{9-11}
    ~ & Cora & CiteSeer & PubMed & ogbn-arxiv & Flickr & Telecom & FinV & Cora & CiteSeer & PubMed\\
    \midrule
    GCN & $80.68$ & $70.83$ & $78.97$ & $70.08$ & $0.5188$ & $0.6080$ & $0.4220$ & $0.9198$ & $0.8959$ & $0.9712$\\
    GCN-Dropout & $\underline{83.16}$ & $71.48$ & $\underline{79.13}$ & $\underline{71.16}$ & $\underline{0.5222}$ & $0.6601$ & $0.4526$ & $0.9278$ & $\textbf{0.9107}$ & $\underline{0.9766}$\\
    GCN-DropEdge & $81.69$ & $71.43$ & $79.06$ & $70.88$ & $0.5214$ & $\underline{0.6650}$ & $\underline{0.4729}$ & $\underline{0.9295}$ & $0.9067$ & $0.9762$\\
    GCN-DropNode & $83.04$ & $\textbf{72.12}$ & $79.00$ & $70.98$ & $0.5213$ & $0.6243$ & $0.4571$ & $0.9238$ & $0.9052$ & $0.9748$\\
    GCN-DropMessage & $\textbf{83.33}$ & $\underline{71.83}$ & $\textbf{79.20}$ & $\textbf{71.27}$ & $\textbf{0.5223}$ & $\textbf{0.6710}$ & $\textbf{0.4876}$ & $\textbf{0.9305}$ & $\underline{0.9071}$ & $\textbf{0.9772}$\\
    \midrule
    GAT & $81.35$ & $70.14$ & $77.20$ & $70.32$ & $0.4988$ & $0.7050$ & $0.4467$ & $0.9118$ & $0.8895$ & $0.9464$\\
    GAT-Dropout & $\textbf{82.41}$ & $71.31$ & $\textbf{78.31}$ & $\textbf{71.28}$ & $0.4998$ & $0.7382$ & $0.4539$ & $0.9182$ & $0.9055$ & $0.9536$\\
    GAT-DropEdge & $81.82$ & $71.17$ & $77.70$ & $70.67$ & $\underline{0.5004}$ & $\underline{0.7568}$ & $\textbf{0.4896}$ & $0.9206$ & $0.9037$ & $0.9493$\\
    GAT-DropNode & $82.08$ & $\underline{71.44}$ & $77.98$ & $70.96$ & $0.4992$ & $0.7214$ & $0.4647$ & $\textbf{0.9224}$ & $\textbf{0.9104}$ & $\textbf{0.9566}$\\
    GAT-DropMessage & $\underline{82.20}$ & $\textbf{71.48}$ & $\underline{78.14}$ & $\underline{71.13}$ & $\textbf{0.5013}$ & $\textbf{0.7574}$ & $\underline{0.4861}$ & $\underline{0.9216}$ & $\underline{0.9076}$ & $\underline{0.9553}$\\
    \midrule
    APPNP & $81.45$ & $70.62$ & $79.79$ & $69.11$ & $0.5047$ & $0.6217$ & $0.3952$ & $0.9058$ & $0.8844$ & $0.9531$\\
    APPNP-Dropout & $82.23$ & $71.93$ & $\underline{79.92}$ & $\underline{69.36}$ & $0.5055$ & $0.6578$ & $0.4023$ & $0.9119$ & $0.9071$ & $0.9611$\\
    APPNP-DropEdge & $\textbf{82.75}$ & $\underline{72.10}$ & $79.83$ & $69.15$ & $\underline{0.5061}$ & $\underline{0.6591}$ & $0.4149$ & $\underline{0.9139}$ & $\underline{0.9131}$ & $\underline{0.9626}$\\
    APPNP-DropNode & $81.79$ & $71.50$ & $79.81$ & $69.27$ & $0.5053$ & $0.6412$ & $\underline{0.4182}$ & $0.9068$ & $0.8979$ & $0.9561$\\
    APPNP-DropMessage & $\underline{82.37}$ & $\textbf{72.65}$ & $\textbf{80.04}$ & $\textbf{69.72}$ & $\textbf{0.5072}$ & $\textbf{0.6619}$ & $\textbf{0.4378}$ & $\textbf{0.9165}$ & $\textbf{0.9141}$ & $\textbf{0.9634}$\\
    \bottomrule
    \end{tabular}
    }
\end{table*}

\vpara{Baseline methods.} We compare our proposed DropMessage with other existing random dropping methods, including Dropout~\cite{Hinton2012ImprovingNN}, DropEdge~\cite{rong2019dropedge}, and DropNode~\cite{feng2020graph}. 
We adopt these dropping methods on various GNNs as the backbone model, and compare their performances on different datasets. 

\vpara{Backbone models.} 
In this paper, we mainly consider three mainstream GNNs as our backbone models: GCN~\cite{kipf2016semi}, GAT~\cite{velivckovic2017graph}, and APPNP~\cite{Klicpera2019PredictTP}.
We take the official practice of these methods while make some minor modifications.
All these backbone models have random dropping modules for different steps in their model implementation. For instance, GAT models perform random dropping after self-attention calculation, while APPNP models perform random dropping at the beginning of each iteration. For a fair comparison, we unify the implementation of random dropping modules in the same step for different backbone models. We fix Dropout, DropEdge, and DropNode on the initial input and fix DropMessage at the start point of the message propagation process.



\subsection{Comparison Results}

Table \ref{tab:result_overall} summarizes the overall results. 
For the node classification task, the performance is measured by accuracy on four public datasets (Cora, CiteSeer, PubMed, ogbn-arxiv).  
As for Flickr and two imbalanced industrial datasets, we employ F1 scores. 
When it comes to the link prediction task, we calculate the AUC values for comparisons. 
Considering the space limitation, the std values of the experimental results are presented in the Appendix.

\vpara{Effect of random dropping methods.} It is observed that random dropping methods consistently outperform GNNs without random dropping in both node classification and link prediction. 
Besides, we see that the effects of random dropping methods vary over different datasets, backbone models, and downstream tasks. 
For example, random dropping methods on APPAP obtain an average accuracy improvement of 1.4\% on CiteSeer, while 0.1\% on PubMed. 
Meanwhile, random dropping methods achieve 2.1\% accuracy improvement for GCN on Cora, while only 0.8\% for GAT. 

\vpara{Comparison of different dropping methods.} 
\noteblue{Our proposed DropMessage works well in all settings, exhibiting its strong adaptability to various scenarios.}
Overall, we have 21 settings under the node classification task, each of which is a combination of different backbone models and datasets (\emph{e.g.}, GCN-Cora). 
It is showed that DropMessage achieves the optimal results in 15 settings, and gets sub-optimal results in the rest. 
As to 9 setttings under the link prediction task, DropMessage achieves the optimal results in 5 settings, and sub-optimal results in the rest. 
Moreover, the stable performance of DropMessage over all datasets compared to other methods is clearly presented. 
Taking DropEdge as the counterexample, it appears strong performance on industrial datasets but demonstrates a clear drop on public ones. 
A reasonable explanation is that the message matrix patterns reserved by distinct mask methods vary from each other as presented in Table \ref{tab:methods}. 
With the favor of its finest-grained dropping strategy, DropMessage obtains smaller inductive bias.
Thus, compared with other methods, DropMessage is more applicable in most scenarios.


\begin{figure*}[htp!]
\centering
\subfloat[MADGap]{
\includegraphics[width=0.30\textwidth]{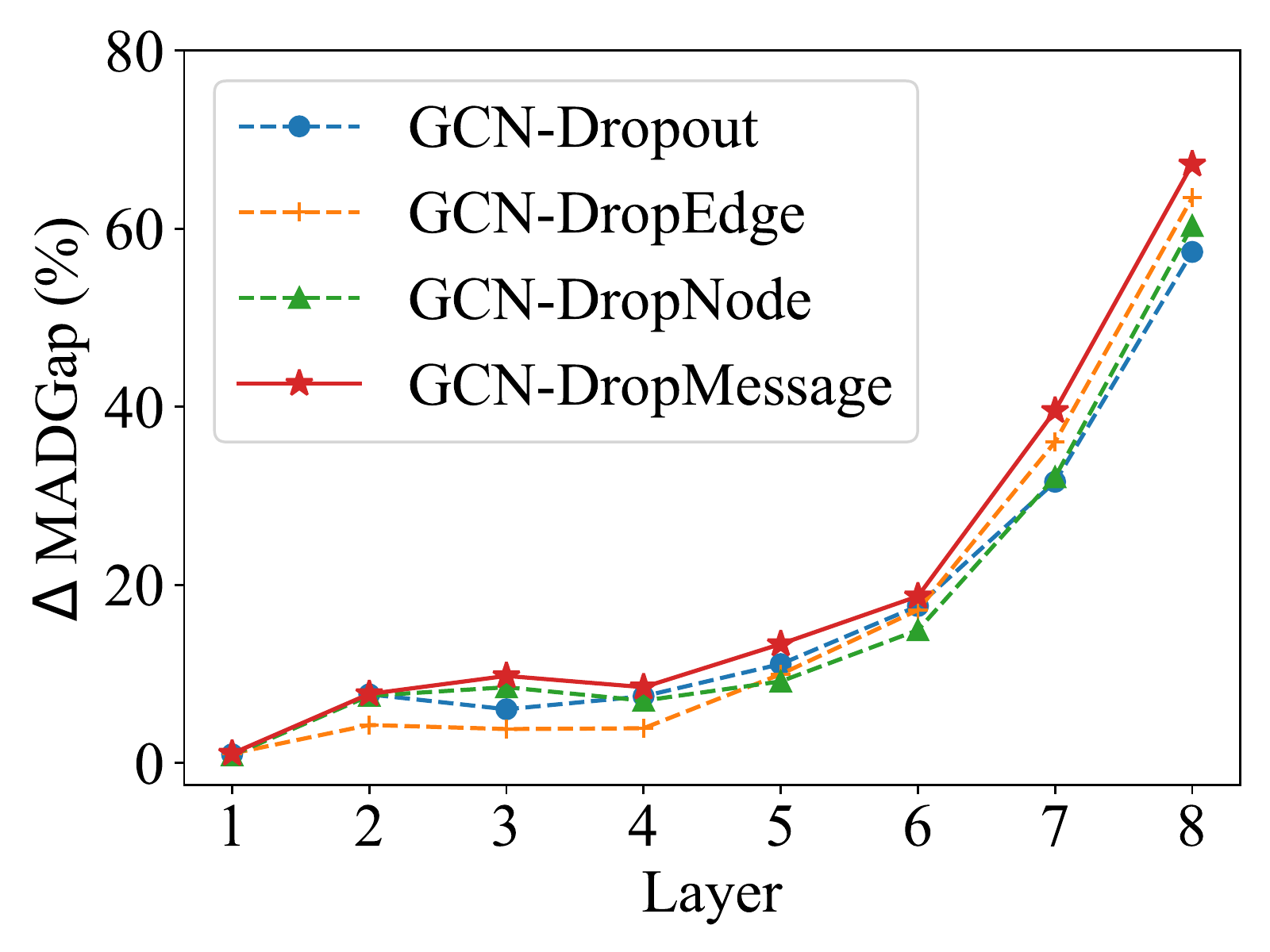}
\label{fig:oversmoothing1}
}
\hspace{0.01\textwidth}
\subfloat[Test Accuracy]{
\includegraphics[width=0.30\textwidth]{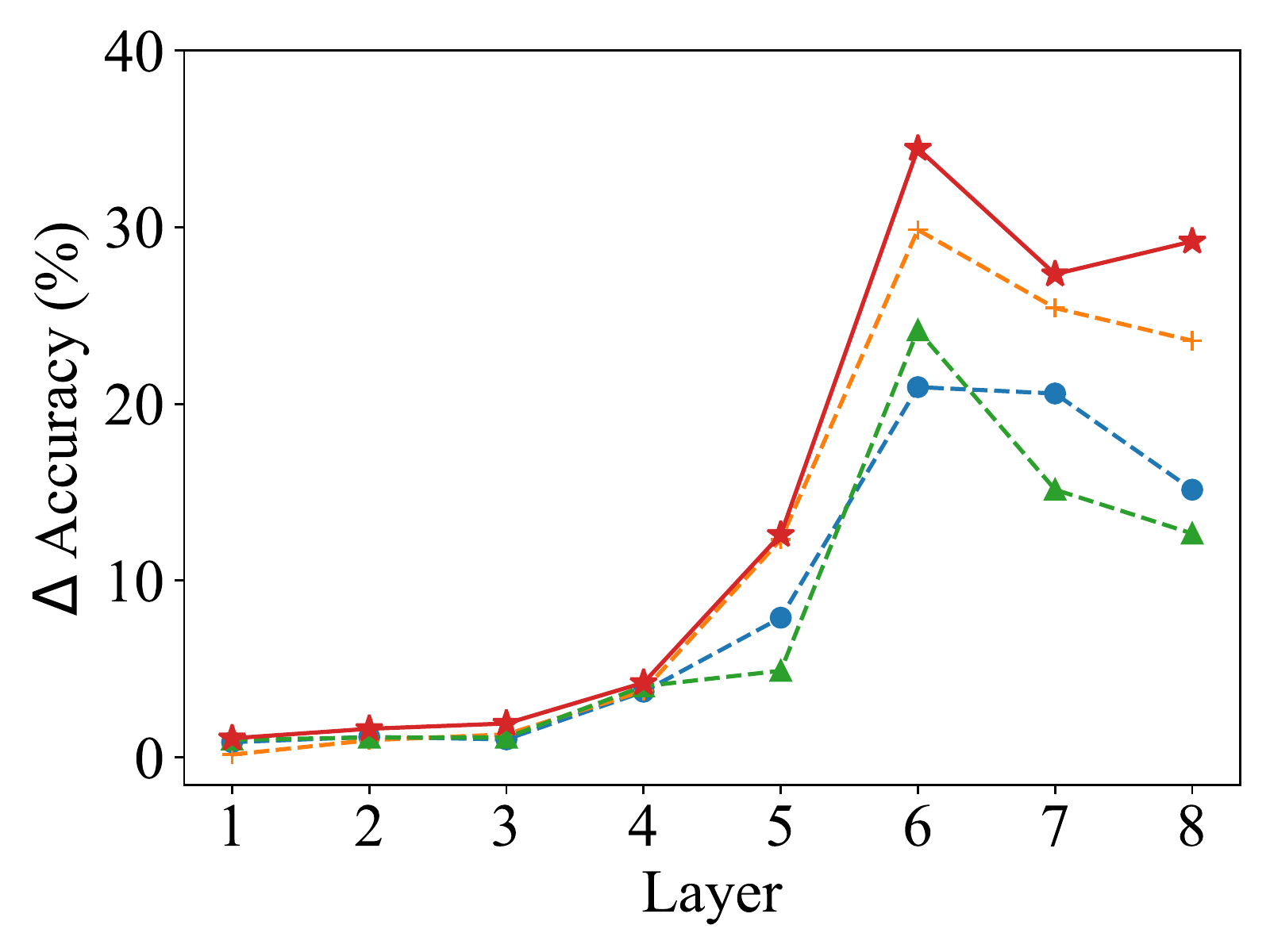}
\label{fig:oversmoothing2}
}
\hspace{0.01\textwidth}
\subfloat[Training Loss]{
\includegraphics[width=0.30\textwidth]{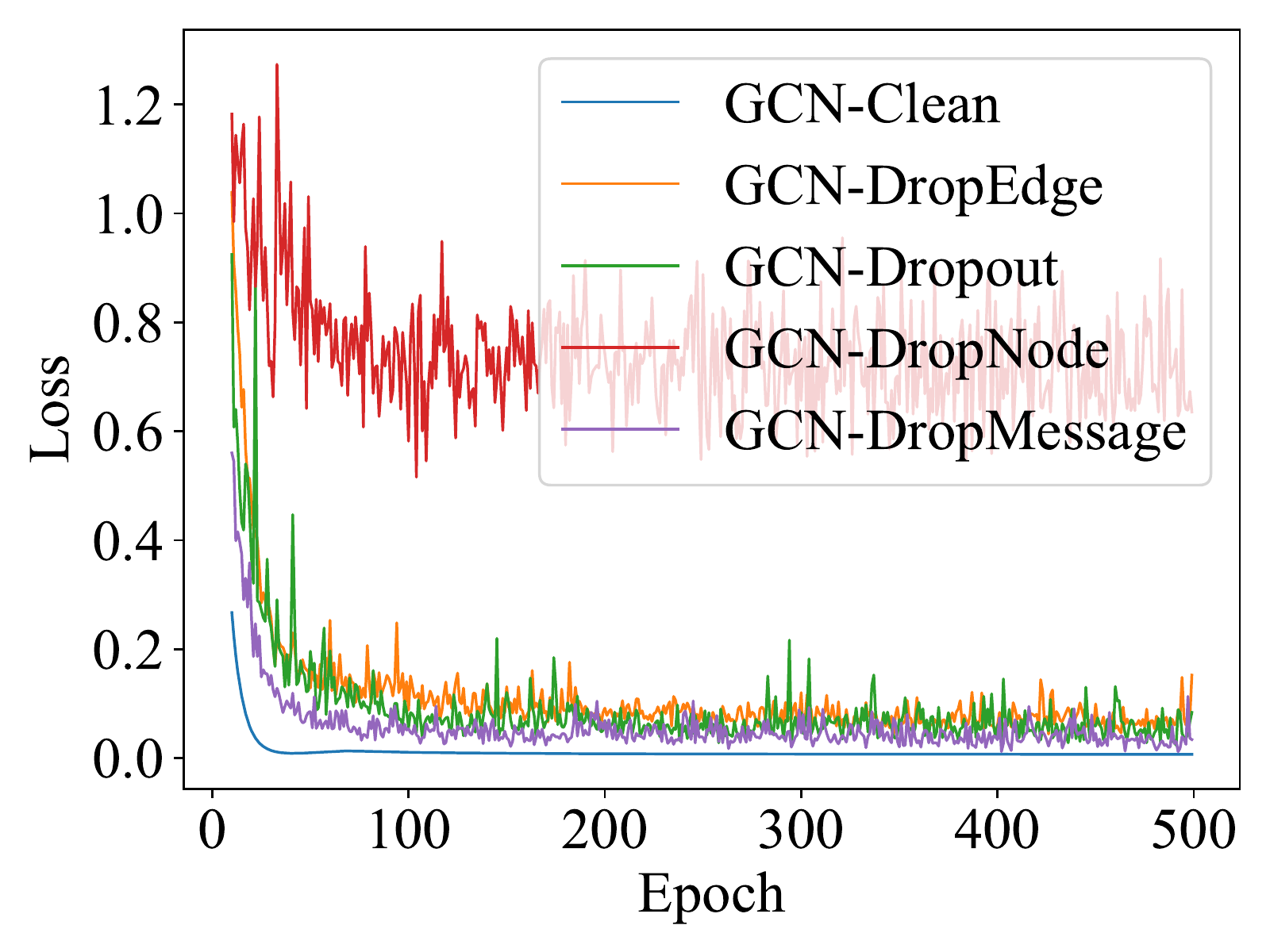}
\label{fig:training_process}
}
\caption{
Over-Smoothing and Training Process Analysis.
}
\end{figure*}

\subsection{Additional Results}
\vpara{Robustness analysis.}
We study the robustness of dropping methods through measuring their ability of handling perturbed graphs. \noteblue{To guarantee that the initial data is comparatively clean}, we conduct experiments on three citation networks: Cora, CiteSeer, and PubMed. We randomly add a certain ratio of edges into these datasets and perform the node classification.
We find that all the random dropping methods have positive effects when the perturbation rate increases from 0\% to 30\%. The average improvement in the case of 30\% perturbation reached 37\% compared to that without perturbation, which indicates that the random dropping methods strengthen the robustness of GNN models. 
Besides, our proposed DropMessage shows its versatility and outperforms other dropping methods in noisy situations. 
\noteblue{Detailed results are exhibited in Appendix}.

\vpara{Over-smoothing analysis.}
Over-smoothing is a common issue on GNNs \cite{Li2018DeeperII}, which implies that the node representations become indistinguishable as the network depth increases.
In this part, we evaluate the effects that various random dropping methods exert on this issue, and measure the degree of over-smoothing by MADGap \cite{Chen2020MeasuringAR}. It should be noted that here a smaller value indicates the more indistinguishable node representations and vice versa.
Experiments are conducted on Cora with GCNs serving as backbone models. 
Figure \ref{fig:oversmoothing1} and Figure \ref{fig:oversmoothing2} show the relative increase of MADGap values and test accuracies of the final node representations compared to the original model without any random dropping techniques.
The results indicate that all these random dropping methods can alleviate over-smoothing by increasing the MADGap values and test accuracies when the depth of the model increases. 
Among all random dropping methods, our proposed DropMessage exhibits a superiority of consistency.
It obtains an average improvement of 3.3\% on MADGap values and an average improvement of 4.9\% on test accuracies compared to other random dropping methods when the layer number $l\geq3$.
This result can be explained by the fact that DropMessage can generate more various messages than other methods, which prevents the nodes from converging to the same representations to some extent.
\modified{A more detailed theoretical explanation can be found in Appendix.}

\vpara{Training process analysis.}
We conduct experiments to analyze the loss during the training process when employing different random dropping methods. 
Figure \ref{fig:training_process} shows the change of loss in GCN training processes when employing different random dropping methods on Cora.
Furthermore, the similar training loss curves can be drawn under other experimental settings.
The experimental results suggest that DropMessage presents the smallest sample variance among all methods, thus achieving the fastest convergence and the most stable performance.
This is consistent with the theoretical results in Section \ref{section42}.

\vpara{Information diversity analysis.}
\begin{table}[t]
\small
\renewcommand\arraystretch{1.5}
    \caption{Classification accuracy (\%) for information diversity analysis (where AVG denotes average, and NW denotes nodewise).}
    \label{tab:result_info}
    \centering
    \resizebox{0.45\textwidth}{!}{
    \begin{tabular}{|c|c|c|c|c|c|c|}
    \hline
    \multirow{2}{*}{Model} 
    & \multicolumn{2}{c|}{GCN} 
    & \multicolumn{2}{c|}{GAT} 
    & \multicolumn{2}{c|}{APPNP}\\
    \cline{2-7}
    ~ & AVG & NW & AVG & NW & AVG & NW\\
    \hline
    Accuracy (\%) & 81.62 & 82.67 & 80.81 & 81.61 & 80.71 & 81.56\\
    \hline
    \end{tabular}
    }
\normalsize
\end{table}
We conduct experiments to evaluate the importance of information diversity for message-passing GNN models. We set Cora as our experimental dataset, which contains 2708 nodes and 5429 edges. 
The average node degree of Cora is close to 4.
According to Equation \ref{equation7}, the upper bound of dropping rate is calculated from the node degree and the feature dimension. The feature dimension number of Cora is 1433, which is much larger than the number of node degree. Therefore, the upper bound is only determined by the degree of the node. 
In (backbone)-nodewise settings, we set the dropping rate to be equal to its upper bound $\delta_{i}=1-\frac{1}{d_{i}}$ for each node. In (backbone)-average settings, we set the dropping rate $\delta_{i}=0.75+\epsilon_{i}$, where $\epsilon_{i} \sim Uniform (-0.15,0.15)$. 
Both of these settings employ DropMessage. 
The average random dropping rate of all nodes is almost identical under these two settings, but only the former one can keep the information diversity in expectation.
Table \ref{tab:result_info} presents the results.
The (backbone)-nodewise settings outperform (backbone)-average settings regardless of which backbone model is selected. 

\section{Conclusion}
In this paper, we propose DropMessage, a general random dropping method for message-passing GNN models. 
We first unify all random dropping methods to our framework via performing dropping on the message matrix and analyzing their effects. 
Then we illustrate the superiority of DropMessage theoretically in stabilizing the training process and keeping information diversity.
Due to its fine-grained dropping operations on the message matrix, DropMessage shows greater applicability in most cases.
By conducting experiments for multiple tasks on five public datasets and two industrial datasets, we demonstrate the effectiveness and generalization of our proposed method.

\section{Acknowledgments}
This work was partially supported by NSFC (62176233), the National Key Research and Development Project of China (2018AAA0101900), and the Fundamental Research Funds for the Central Universities.

\bibliography{aaai23.bib}

\onecolumn{\section{Appendix}

\subsection{Derivation Details}








\vpara{Detailed proof of Theorem 1.}

\vpara{Theorem 1} \textit{Unbiased random dropping methods introduce an extra regularization term into the objective functions, which make the models more robust.}

We give more derivation details of Theorem 1.
When we use cross-entropy as the loss function, the objective function can be expressed as follows:
\begin{align}
L_{CE}&=-\sum_{j,y_{j}=1}log(z_{j})-\sum_{k,y_{k}=0}log(1-z_{k})\nonumber \\
&=\sum_{j,y_{j}=1}log(1+e^{-h_{j}})+\sum_{k,y_{k}=0}log(1+e^{h_{k}})\nonumber
\end{align}

According to above equation, the initial objective function is $L_{CE}=\sum_{j,y_{j}=1}log(1+e^{-h_{j}})+\sum_{k,y_{k}=0}log(1+e^{h_{k}})$.
When we perturb the message matrix, the objective function can be regarded as a process of adding a bias to the original function, expressed as follows:
\begin{align}
E(\widetilde{L}_{CE})&=\sum_{j,y_{j}=1}[log(1+e^{-h_{j}})+E(f(\tilde{h}_{j},h_{j}))]\nonumber\\
&+\sum_{k,y_{k}=0}[log(1+e^{h_{k}})+E(g(\tilde{h}_{k},h_{k}))]\nonumber
\end{align}
where $f(\tilde{h}_{j},h_{j})=log(1+e^{-\tilde{h}_{j}})-log(1+e^{-h_{j}})$, and $g(\tilde{h}_{k},h_{k})=log(1+e^{\tilde{h}_{k}})-log(1+e^{h_{k}})$.
We can approximate it with the second-order Taylor expansion of $f(.)$ and $g(.)$ around $h_{i}$. Thus, the objective function in expectation can be expressed as bellow:
\begin{align}
&E(\widetilde{L}_{CE})=L_{CE}\nonumber \\
&+E(\sum_{j,y_{j}=1}[(-1+z_{j})(\tilde{h}_j-h_{j})+\frac{1}{2}z_{j}(1-z_{j})(\tilde{h}_{j}-h_{j})^{2}]) \nonumber \\
&+E(\sum_{k,y_{k}=0}[z_{k}(\tilde{h}_{k}-h_{k})+\frac{1}{2}z_{k}(1-z_{k})(\tilde{h}_{k}-h_{k})^{2}]) \nonumber \\
&=L_{CE}+\sum_{i}\frac{1}{2}z_{i}(1-z_{i})Var(\tilde{h}_{i})\nonumber
\end{align}

\vpara{Proof of Theorem 2.}

\vpara{Theorem 2} \textit{DropMessage presents the smallest sample variance among all existing random dropping methods on message-passing GNNs with the same dropping rate $\delta$.}

\begin{proof}
As stated in Lemma 1, all random dropping methods on graphs can be converted to masking operations on the message matrix $\mathbf{M}$. 
We can measure the difference of message matrices in different epochs by the way of comparing the sample variance of random dropping methods, which can be measured via the norm variance of the message matrix $|\mathbf{M}|_{F}$.
Without loss of generality, we assume the original message matrix $\mathbf{M}$ is $\mathbf{1}_{n\times n}$, \emph{i.e.}, every element is $1$. Thus, we can calculate its sample variance via the $1$-norm of the message matrix.

We consider that the message-passing GNNs do not possess the node-sampler or the edge-sampler, which means every directed edge corresponds to a row vector in the message matrix $\mathbf{M}$. 
For analytical simplicity, we assume that the graph is undirected and the degree of each node is $d$.
In this case, $k = 2|E| = nd$ rows of the message matrix counts in total.
All random dropping methods can be considered as multiple independent Bernoulli samplings. The whole process conforms to a binomial distribution, and so we can calculate the variance of $|\mathbf{M}|$.

\vpait{Dropout.} 
Perform $nc$ times of Bernoulli sampling.
Dropping an element in the feature matrix leads to masking $d$ elements in the message matrix. 
Its variance can be calculated by $Var_{do}(\mathbf{|M|})=(1-\delta)\delta ncd^{2}$.

\vpait{DropEdge.} 
Perform $\frac{nd}{2}$ times of Bernoulli sampling. 
Dropping an element in the adjacency matrix leads to masking $2c$ elements in the message matrix. 
Its variance can be calculated by $Var_{de}(\mathbf{|M|})=2(1-\delta)\delta nc^{2}d$.

\vpait{DropNode.} 
Perform $n$ times of Bernoulli sampling. 
Dropping an element in the node set leads to masking $cd$ elements in the message matrix. 
Its variance can be calculated by $Var_{dn}(\mathbf{|M|})=(1-\delta)\delta nc^{2}d^{2}$.

\vpait{DropMessage.} 
Perform $ncd$ times of Bernoulli sampling. 
Dropping an element in the message matrix leads to masking $1$ elements in the message matrix. 
Its variance can be calculated by $Var_{dm}(\mathbf{|M|})=(1-\delta)\delta ncd$.

Therefore, the variances of the random dropping methods are sorted as follows:
\begin{align}
Var_{dm}(|\mathbf{M}|)\leq Var_{do}(|\mathbf{M}|)\leq Var_{dn}(|\mathbf{M}|) \nonumber
\end{align}
\vspace{-0.5cm}
\begin{align}
Var_{dm}(|\mathbf{M}|)\leq Var_{de}(|\mathbf{M}|) \nonumber
\end{align}
Our DropMessage has the smallest sample variance among all existing random dropping methods.
\end{proof}

\subsection{Experiment Details}

\vpara{Hardware spcification and environment.} 
We run our experiments on the machine with Intel Xeon Gold CPUs (6240@2.6Ghz), ten NVIDIA GeForce 2080ti GPUs (11GB).
The code is written in Python 3.8, Pytorch 1.8, and Pytorch-Geometric 1.7.



\vpara{Dataset statistics.} Table \ref{tab:statistics} shows the statistics of datasets.

\begin{table*}[h]
	\centering
	\caption{Dataset Statistics.}
	\label{tab:statistics}
	\begin{tabular}{cccccc}
	\toprule
	Dataset	& Nodes	& Edges	& Feature & Classes & Train/Val/Test\\
	\midrule
	Cora & 2708 & 5429 & 1433 & 7 & 140 / 500 / 1000\\
    CiteSeer & 3327 & 4732 & 3703 & 6 & 120 / 500 / 1000\\ 
	PubMed & 19717 & 44338 & 500 & 3 & 60 / 500 / 1000\\
	ogbn-arxiv & 169343 & 1166243 & 128 & 40 & 90941 / 29799 / 48603\\ 
	Flickr & 89250 & 899756 & 500 & 7 & 50\% / 25\% / 25\%\\
	FinV & 340751 & 1575498 & 261 & 2 & 60\% / 20\% / 20\%\\
	Telecom & 509304 & 809996 & 21 & 2 & 60\% / 20\% / 20\%\\
	\bottomrule
    \end{tabular}
\end{table*}



\vpara{Implementation details.} We conduct 20 independent experiments for each setting and obtain the average results.
On the five public datasets, we continue to employ the same hyper-parameter settings as previous works have proposed. And on the two real-world datasets, we obtain the best parameters through careful tuning. 
As for public datasets \textit{Cora}, \textit{CiteSeer}, \textit{PubMed}, and \textit{Flickr}, we apply two-layer models. 
However, when it comes to the public dataset \textit{ogbn-arxiv} and two industrial datasets, \textit{Telecom} and \textit{FinV}, we employ three-layer models with two batch normalization layers between the network layers.
These experimental settings are identical for node classification tasks and link prediction tasks.
The number of hidden units on GCNs is 16 for \textit{Cora}, \textit{CiteSeer}, \textit{PubMed}, and is 64 for others. For GATs, we apply eight-head models with 8 hidden units for \textit{Cora}, \textit{CiteSeer}, \textit{PubMed}, and use single-head models with 128 hidden units for the other datasets. As for APPNPs, we use the teleport probability $\alpha=0.1$ and $K=10$ power iteration steps. The number of hidden units on APPNPs is always 64 for all datasets.
In all cases, we use Adam optimizers with learning rate of $0.005$ and $L_{2}$ regularization of $5\times10^{-4}$, and train each model 200 epochs.
 We adjust the dropping rate from 0.05 to 0.95 in steps of 0.05 and select the optimal one for each setting. 
Table \ref{tab:drop_rate1} and \ref{tab:drop_rate2} present the optimal selections of dropping rates $\delta$.

\begin{table*}[h]
	\centering
	\caption{Statistics of Optimal Dropping Rates $\delta$ for Link Predictions.}
	\label{tab:drop_rate1}
	\begin{tabular}{c|ccccc}
	\toprule
	Dataset	& Dropout & DropEdge & DropNode & DropMessage\\
	\midrule
	GCN-Cora & 0.20 & 0.45 & 0.40 & 0.40\\
    GAT-Cora & 0.20 & 0.40 & 0.15 & 0.45\\
    APPNP-Cora & 0.50 & 0.35 & 0.05 & 0.35\\
    \midrule
    GCN-CiteSeer & 0.25 & 0.30 & 0.25 & 0.50\\
    GAT-CiteSeer & 0.25 & 0.50 & 0.15 & 0.50\\
    APPNP-CiteSeer & 0.50 & 0.50 & 0.10 & 0.40\\
    \midrule
    GCN-PubMed & 0.25 & 0.25 & 0.50 & 0.25\\
    GAT-PubMed & 0.50 & 0.15 & 0.25 & 0.50\\
    APPNP-PubMed & 0.20 & 0.25 & 0.05 & 0.20\\
	\bottomrule
    \end{tabular}
\end{table*}

\begin{table*}[h]
	\centering
	\caption{Statistics of Optimal Dropping Rates $\delta$ for Node Classifications.}
	\label{tab:drop_rate2}
	\begin{tabular}{c|ccccc}
	\toprule
	Dataset	& Dropout & DropEdge & DropNode & DropMessage\\
	\midrule
	GCN-Cora & 0.90 & 0.40 & 0.90 & 0.90\\
    GAT-Cora & 0.90 & 0.40 & 0.25 & 0.90\\
    APPNP-Cora & 0.85 & 0.70 & 0.15 & 0.80\\
    \midrule
    GCN-CiteSeer & 0.80 & 0.40 & 0.85 & 0.90\\
    GAT-CiteSeer & 0.35 & 0.90 & 0.45 & 0.90\\
    APPNP-CiteSeer & 0.85 & 0.45 & 0.35 & 0.80\\
    \midrule
    GCN-PubMed & 0.20 & 0.10 & 0.10 & 0.15\\
    GAT-PubMed & 0.60 & 0.70 & 0.35 & 0.85\\
    APPNP-PubMed & 0.80 & 0.50 & 0.10 & 0.75\\
    \midrule
    GCN-ogbn-arxiv & 0.20 & 0.20 & 0.15 & 0.25\\
    GAT-ogbn-arxiv & 0.20 & 0.20 & 0.15 & 0.20\\
    APPNP-ogbn-arxiv & 0.20 & 0.20 & 0.20 & 0.25\\
    \midrule
    GCN-Flickr & 0.15 & 0.20 & 0.15 & 0.20\\
    GAT-Flickr & 0.35 & 0.15 & 0.05 & 0.40\\
    APPNP-Flickr & 0.35 & 0.35 & 0.05 & 0.35\\
	\bottomrule
    \end{tabular}
\end{table*}

\vpara{Additional data for comparison results.}
For the node classification task, the performance is measured by accuracy on four public datasets (Cora, CiteSeer, PubMed, ogbn-arxiv).  
As for Flickr and two imbalanced industrial datasets, we employ F1 scores. 
When it comes to the link prediction task, we calculate the AUC values for comparisons.
Table \ref{tab:result_overall_std} present the std values of comparison results.
\begin{table*}[h]
    \centering
    \caption{The Std Values of Comparison Results.
    }
    \label{tab:result_overall_std}
    \resizebox{1\textwidth}{!}{
    \begin{tabular}{c|cccc||ccc|||ccc}
    \toprule
    \multirow{2}{*}{\diagbox [width=11em, trim=lr] {Model}{Task \& Dataset}} & \multicolumn{7}{c|||}{Node classification} & \multicolumn{3}{c}{Link prediction}\\
    \cmidrule{2-8} \cmidrule{9-11}
    ~ & Cora & CiteSeer & PubMed & ogbn-arxiv & Flickr & Telecom & FinV & Cora & CiteSeer & PubMed\\
    \midrule
    GCN & $0.37$ & $0.55$ & $0.67$ & $0.41$ & $0.0044$ & $0.0043$ & $0.0061$ & $0.0022$ & $0.0053$ & $0.0043$\\
    GCN-Dropout & $0.68$ & $0.59$ & $0.77$ & $0.52$ & $0.0050$ & $0.0067$ & $0.0069$ & $0.0044$ & $0.0073$ & $0.0055$\\
    GCN-DropEdge & $0.91$ & $0.72$ & $0.81$ & $0.54$ & $0.0061$ & $0.0055$ & $0.0077$ & $0.0045$ & $0.0077$ & $0.0080$\\
    GCN-DropNode & $1.06$ & $0.95$ & $0.88$ & $0.75$ & $0.0064$ & $0.0079$ & $0.0089$ & $0.0064$ & $0.0091$ & $0.0094$\\
    GCN-DropMessage & $0.59$ & $0.59$ & $0.53$ & $0.56$ & $0.0042$ & $0.0027$ & $0.0055$ & $0.0044$ & $0.0064$ & $0.0051$\\
    \midrule
    GAT & $0.59$ & $0.62$ & $0.53$ & $0.37$ & $0.0042$ & $0.0033$ & $0.0043$ & $0.0042$ & $0.0047$ & $0.0049$\\
    GAT-Dropout & $0.77$ & $0.71$ & $0.67$ & $0.49$ & $0.0063$ & $0.0046$ & $0.0067$ & $0.0054$ & $0.0062$ & $0.0066$\\
    GAT-DropEdge & $0.82$ & $1.02$ & $0.79$ & $0.54$ & $0.0077$ & $0.0052$ & $0.0055$ & $0.0079$ & $0.0082$ & $0.0087$\\
    GAT-DropNode & $0.89$ & $0.86$ & $0.88$ & $0.62$ & $0.0079$ & $0.0077$ & $0.0091$ & $0.0093$ & $0.0092$ & $0.0085$\\
    GAT-DropMessage & $0.69$ & $0.67$ & $0.47$ & $0.54$ & $0.0052$ & $0.0055$ & $0.0056$ & $0.0037$ & $0.0032$ & $0.0077$\\
    \midrule
    APPNP & $0.33$ & $0.34$ & $0.51$ & $0.33$ & $0.0024$ & $0.0032$ & $0.0053$ & $0.0038$ & $0.0017$ & $0.0069$\\
    APPNP-Dropout & $0.43$ & $0.54$ & $0.43$ & $0.29$ & $0.0032$ & $0.0059$ & $0.0072$ & $0.0063$ & $0.0046$ & $0.0042$\\
    APPNP-DropEdge & $0.72$ & $0.71$ & $0.86$ & $0.55$ & $0.0039$ & $0.0032$ & $0.0088$ & $0.0087$ & $0.0058$ & $0.0098$\\
    APPNP-DropNode & $0.49$ & $0.42$ & $0.66$ & $0.73$ & $0.0054$ & $0.0065$ & $0.0077$ & $0.0074$ & $0.0035$ & $0.0103$\\
    APPNP-DropMessage & $0.52$ & $0.24$ & $0.37$ & $0.38$ & $0.0044$ & $0.0039$ & $0.0044$ & $0.0065$ & $0.0041$ & $0.0032$\\
    \bottomrule
    \end{tabular}
    }
\end{table*}

\vpara{Results of robustness analysis.}
We conduct experiments for robustness analysis on three citation networks: Cora, CiteSeer, and PubMed.
Specifically, we randomly add a certain ratio of edges (0\%, 10\%, 20\%, 30\%) into these datasets and perform the node classification.
Table \ref{tab:result_perturb} summarizes the classification accuracy of robustness analysis.

\begin{table*}[htbp]
    \centering
    \caption{Classification Accuracy (\%) for Robustness Analysis.}
    \label{tab:result_perturb}
    \resizebox{0.98\textwidth}{!}{
    \begin{tabular}{c|cccc|cccc|cccc}
    \toprule
    \multirow{2}{*}{{\diagbox [width=8em, trim=lr] {Model}{Dataset}}} 
    & \multicolumn{4}{c}{Cora} 
    & \multicolumn{4}{c}{CiteSeer} 
    & \multicolumn{4}{c}{PubMed}\\
    \cmidrule(lr){2-13} 
    ~ & 0\% & 10\% & 20\% & 30\% & 0\% & 10\% & 20\% & 30\% & 0\% & 10\% & 20\% & 30\%\\
    \midrule
    GCN & 80.68 & 78.51 & 76.72 & 75.36 & 70.83 & 68.66 & 66.32 & 65.15 & 78.97 & 75.55 & 73.18 & 72.11\\
    GCN-Dropout & \underline{83.16} & \underline{80.97} & \underline{78.17} & \underline{76.83} & 71.48 & 69.86 & 67.35 & 66.08 & \underline{79.13} & \underline{76.94} & \underline{74.94} & \underline{74.07}\\
    GCN-DropEdge & 81.69 & 79.45 & 77.47 & 76.44 & 71.43 & 69.60 & 67.26 & 66.14 & 79.06 & 76.57 & 74.88 & 73.93\\
    GCN-DropNode & 83.04 & 80.13 & 78.12 & 76.72 & \textbf{72.12} & \textbf{70.51} & \textbf{68.21} & \textbf{66.94} & 79.00 & 76.74 & 74.71 & 73.86\\
    GCN-DropMessage & \textbf{83.33} & \textbf{81.04} & \textbf{79.09} & \textbf{77.26} & \underline{71.83} & \underline{70.08} & \underline{67.61} & \underline{66.49} & \textbf{79.20} & \textbf{77.10} & \textbf{75.02} & \textbf{74.11}\\
    \midrule
    GAT & 81.35 & 78.14 & 76.48 & 74.56 & 70.14 & 67.51 & 64.99 & 63.65 & 77.20 & 75.05 & 72.81 & 71.59\\
    GAT-Dropout & \textbf{82.41} & \textbf{80.20} & \textbf{78.71} & \textbf{77.23} & 71.31 & 68.38 & 66.84 & 64.92 & \textbf{78.31} & \underline{76.05} & \underline{74.14} & \underline{72.88}\\
    GAT-DropEdge & 81.82 & 79.08 & 76.92 & 75.32 & 71.17 & \underline{69.07} & \underline{67.21} & \underline{65.31} & 77.70 & 75.92 & 74.02 & 72.73\\
    GAT-DropNode & 82.08 & 78.80 & 76.98 & 75.84 & \underline{71.44} & 68.57 & 66.42 & 64.68 & 77.98 & 75.87 & 73.57 & 72.38\\
    GAT-DropMessage & \underline{82.20} & \underline{79.70} & \underline{78.11} & \underline{76.53} & \textbf{71.48} & \textbf{69.24} & \textbf{67.47} & \textbf{65.49} & \underline{78.14} & \textbf{76.20} & \textbf{74.22} & \textbf{72.97}\\
    \midrule
    APPNP & 81.45 & 77.75 & 75.61 & 73.54 & 70.62 & 65.76 & 62.60 & 60.92 & 79.79 & 75.29 & 72.77 & 71.06\\
    APPNP-Dropout & 82.23 & 79.06 & 76.55 & 74.30 & 71.93 & 66.55 & 63.22 & 61.61 & \underline{79.92} & 76.45 & 74.12 & 72.17\\
    APPNP-DropEdge & \textbf{82.75} & \underline{78.90} & \textbf{76.63} & \textbf{74.68} & \underline{72.10} & \underline{66.58} & \textbf{63.27} & \textbf{61.77} & 79.83 & \underline{76.72} & \underline{74.17} & \underline{72.21}\\
    APPNP-DropNode & 81.79 & 78.17 & 75.79 & 73.76 & 71.50 & 65.86 & 63.01 & 61.01 & 79.81 & 76.45 & 74.05 & 72.01\\
    APPNP-DropMessage & \underline{82.37} & \textbf{79.12} & \underline{76.60} & \underline{74.59} & \textbf{72.65} & \textbf{66.74} & \underline{63.25} & \underline{61.59} & \textbf{80.04} & \textbf{76.73} & \textbf{74.25} & \textbf{72.25}\\
    \bottomrule
    \end{tabular}
    }
\end{table*}

\subsection{Related Works About Data Augmentations}
In this section, we introduce some previous works about data augmentation techniques that are related to the random dropping methods discussed in our paper.
GNN's effectiveness tends to be weakened due to the noise and low-resource problems in real-world graph data~\cite{dai2021nrgnn}\cite{ding2022meta}\cite{sun2020multi}. 
Data augmentation has attracted a lot of research interest as it is an effective tool to improve model performance in noisy settings~\cite{zhao2021data}. 
However, apart from i.i.d. data, graph data, which is defined on non-Euclidean space with multi-modality, is hard to be handled by conventional data augmentation methods~\cite{ding2022data}. To address this problem, an increasing number of graph data augmentation methods have been proposed, which include feature-wise~\cite{velickovic2019deep}, structure-wise~\cite{cai2021graph}\cite{jin2021graph}, and label-wise augmentations~\cite{zhang2017mixup}\cite{verma2019manifold}.

\subsection{Additional Experiments}
\vpara{Results on SOTA models.}
We also compare the performance of different random dropping methods and our proposed DropMessage on two SOTA backbone models: DAGNN~\cite{liu2020towards} and GCNII~\cite{chenWHDL2020gcnii}.
We employ these two models to perform node classification tasks on five datasets.

\begin{table*}[h!]
    \centering
    \caption{Results on SOTA Models.}
    \label{tab:backbone}
    \resizebox{0.98\textwidth}{!}{
    \begin{tabular}{c|c|c|c|c|c}
    \toprule
    {\diagbox [width=8em, trim=lr] {Model}{Dataset}}
    & Cora & CiteSeer & PubMed & ogbn-arxiv & Flickr\\
    \midrule
    DAGNN & $82.73\pm0.27$ & $72.82\pm0.43$ & $80.37\pm0.32$ & $71.60\pm0.19$ & $0.5243\pm0.0017$ \\
    DAGNN-Dropout & $84.50\pm0.50)$ & $73.30\pm0.58$ & $80.56\pm0.52$ & $72.09\pm0.25$ & $0.5324\pm0.0031$ \\
    DAGNN-DropEdge & $83.78\pm0.55$ & $73.02\pm0.50$ & $80.27\pm0.42$ & $71.88\pm0.21$ & $0.5311\pm0.0029$ \\
    DAGNN-DropNode & $83.97\pm0.45$ & $73.41\pm0.67$ & $80.42\pm0.40$ & $72.03\pm0.27$ & $0.5304\pm0.0016$ \\
    DAGNN-DropMessage & $\mathbf{84.64\pm0.61}$ & $\mathbf{73.35\pm0.69}$ & $\mathbf{80.58\pm0.58}$ & $\mathbf{72.23\pm0.33}$ & $\mathbf{0.5325\pm0.0024}$ \\
    \midrule
    GCNII & $82.24\pm0.15$ & $72.11\pm0.41$ & $79.85\pm0.20$ & $72.37\pm0.13$ & $0.5143\pm0.0032$ \\
    GCNII-Dropout & $85.45\pm0.50$ & $73.42\pm0.55$ & $80.18\pm0.37$ & $72.74\pm0.16$ & $0.5170\pm0.0029$ \\
    GCNII-DropEdge & $84.96\pm0.46$ & $72.98\pm0.42$ & $80.05\pm0.29$ & $72.41\pm0.15$ & $0.5180\pm0.0023$ \\
    GCNII-DropNode & $85.15\pm0.66$ & $73.44\pm0.48$ & $79.87\pm0.21$ & $72.38\pm0.21$ & $0.5159\pm0.0029$ \\
    GCNII-DropMessage & $\mathbf{85.53\pm0.62}$ & $\mathbf{73.28\pm0.51}$ & $\mathbf{80.21\pm0.33}$ & $\mathbf{72.78\pm0.16}$ & $\mathbf{0.5192\pm0.0026}$ \\
    \bottomrule
    \end{tabular}
    }
\end{table*}

Table \ref{tab:backbone} presents the experimental results.
The results indicate that DropMessage consistently outperforms other random dropping methods.

\vpara{Comparison to Random Augmentation Methods.}
Random dropping methods are similar to random augmentation techniques used in graph contrastive learning.
We compare the performance of our proposed DropMessage with some widely-used augmentation techniques~\cite{Ding2022DataAF}, and their brief descriptions are listed as below.

\textit{Node Dropping}: it randomly discards a certain portion of vertices along with their connections.

\textit{Edge Perturbation}: it perturbs the connectivities in graph through randomly adding or dropping a certain ratio of edges.

\textit{Subgraph}: it samples a subgraph using random walk.

\begin{table*}[h!]
    \centering
    \caption{Results of Random Augmentation Methods.}
    \label{tab:add_methods}
    \resizebox{0.98\textwidth}{!}{
    \begin{tabular}{c|c|c|c|c|c}
    \toprule
    {\diagbox [width=8em, trim=lr] {Model}{Dataset}}
    & Cora & CiteSeer & PubMed & ogbn-arxiv & Flickr\\
    \midrule
    GCN & $80.68\pm0.37$ & $70.83\pm0.55$ & $78.97 \pm0.67$ & $70.08 \pm0.41$ & $0.5188 \pm0.0044$ \\
    GCN-DropMessage & $\mathbf{83.33\pm0.59}$ & $\mathbf{71.83\pm0.59}$ & $\mathbf{79.20 \pm0.53}$ & $\mathbf{71.27\pm0.56}$ & $\mathbf{0.5223\pm0.0042}$ \\
    GCN-NodeDropping & $81.06 \pm0.53$ & $71.20 \pm0.33$ & $79.06 \pm0.25$ & $70.91 \pm0.54$ & $0.5210 \pm0.0065$ \\
    GCN-EdgePerturbation & $82.12 \pm0.54$ & $71.70 \pm0.41$ & $79.08 \pm0.74$ & $70.69 \pm0.62$ & $0.5205\pm0.0051$ \\
    GCN-Subgraph & $77.40 \pm0.47$ & $68.88 \pm0.46$ & $77.44 \pm0.58$ & $69.34 \pm0.69$ & $0.5087 \pm0.0061$ \\
    \midrule
    GAT & $81.35 \pm0.59$ & $70.14 \pm0.62$ & $77.20 \pm0.53$ & $70.32 \pm0.37$ & $0.4988\pm0.0042$ \\
    GAT-DropMessage & $\mathbf{82.20 \pm0.69}$ & $\mathbf{71.48 \pm0.67}$ & $\mathbf{78.14 \pm0.47}$ & $\mathbf{71.13\pm0.54}$ & $\mathbf{0.5013\pm0.0052}$ \\
    GAT-NodeDropping & $80.68 \pm0.58$ & $70.42 \pm1.77$ & $78.08\pm0.22$ & $70.80 \pm0.67$ & $0.4990 \pm0.0051$ \\
    GAT-EdgePerturbation & $81.92 \pm0.60$ & $70.52 \pm1.02$ & $78.04\pm0.54$ & $70.44 \pm0.47$ & $0.5001\pm0.0046$ \\
    GAT-Subgraph & $77.48 \pm0.41$ & $69.20 \pm0.34$ & $78.04 \pm0.21$ & $68.56 \pm0.88$ & $0.4914 \pm0.0060$ \\
    \midrule
    APPNP & $81.45\pm 0.33$ & $70.62 \pm0.34$ & $79.79 \pm0.51$ & $69.11\pm0.33 $& $0.5047\pm0.0024$ \\
    APPNP-DropMessage & $\mathbf{82.37 \pm0.52}$ & $\mathbf{72.65\pm0.24}$ & $\mathbf{80.04\pm0.37}$ & $\mathbf{69.72 \pm0.38}$ & $\mathbf{0.5072\pm0.0044}$ \\
    APPNP-NodeDropping & $81.94\pm0.91$ & $71.36 \pm0.28$ & $79.18 \pm0.27$ & $69.23 \pm0.54$ & $0.5041 \pm0.0058$ \\
    APPNP-EdgePerturbation & $80.10 \pm0.61$ &$ 72.12 \pm0.23 $& $78.90 \pm0.18 $& $69.04 \pm0.35 $& $0.5052\pm0.0043$ \\
    APPNP-Subgraph & $80.28 \pm0.58$ &$69.80\pm0.51$  & $76.64\pm1.57$ & $67.23\pm 1.06$ & $0.4833 \pm0.0089$ \\
    \bottomrule
    \end{tabular}
    }
\end{table*}

Table \ref{tab:add_methods} summaries the experimental results on five datasets (Cora, CiteSeer, PubMed, ogbn-arxiv and Flickr) for node classification tasks.
The results indicate that our proposed DropMessage consistently outperforms the random data augmentation methods.

\vpara{Graph Property Prediction.}
We perform the graph property prediction task on ogbg-molhiv and ogbg-molpcba~\cite{hu2020ogb}.
They are two molecular property prediction datasets adopted from the MoleculeNet~\cite{Wu2017MoleculeNetAB}.
Table \ref{tab:result_graph_classification} shows the results of ROC-AUC scores with GCN as the backbone model.

\begin{table*}[h!]
    \centering
    \caption{ROC-AUC Scores for Graph Property Prediction.}
    \label{tab:result_graph_classification}
    \begin{tabular}{c|c|c}
    \toprule
    {\diagbox [width=8em, trim=lr] {Model}{Dataset}}
    & ogbg-molhiv 
    & ogbg-molpcba\\
    \midrule
    GCN & $0.7581\pm0.0061$ & $0.2013\pm0.0012$ \\
    GCN-Dropout & $0.7606\pm0.0097$ & $0.2020\pm0.0024$ \\
    GCN-DropEdge & $0.7602\pm0.0093$ & $0.2015\pm0.0017$ \\
    GCN-DropNode & $0.7592\pm0.0087$ & $0.2019\pm0.0022$ \\
    GCN-DropMessage & $\mathbf{0.7614\pm0.0103}$ & $\mathbf{0.2027\pm0.0041}$ \\
    \bottomrule
    \end{tabular}
\end{table*}

\vpara{Graph Rewiring.}
We conduct experiments on graph rewiring to evaluate the robustness of random dropping methods.
Specifically, we first remove a certain ratio of edges, and then randomly add an equal number of edges.
We perform the experiments on three citation datasets with GCN as the backbone model.
Table \ref{tab:result_perturb_2} presents the results.

\begin{table*}[h!]
    \centering
    \caption{Classification Accuracy (\%) for Rewiring Graphs.}
    \label{tab:result_perturb_2}
    \resizebox{0.98\textwidth}{!}{
    \begin{tabular}{c|cccc|cccc|cccc}
    \toprule
    \multirow{2}{*}{{\diagbox [width=8em, trim=lr] {Model}{Dataset}}} 
    & \multicolumn{4}{c}{Cora} 
    & \multicolumn{4}{c}{CiteSeer} 
    & \multicolumn{4}{c}{PubMed}\\
    \cmidrule(lr){2-13} 
    ~ & 0\% & 10\% & 20\% & 30\% & 0\% & 10\% & 20\% & 30\% & 0\% & 10\% & 20\% & 30\%\\
    \midrule
    GCN & 80.68 & 77.36 & 74.64 & 71.96 & 70.83 & 68.46 & 65.84 & 62.94 & 78.97 & 76.14 & 72.70 & 71.70 \\
    GCN-Dropout & 83.16 & 79.06 & 75.90 & 72.84 & 71.48 & 69.14 & 66.24 & 63.20 & 79.13 & 76.60 & 73.92 & 72.46 \\
    GCN-DropEdge & 81.69 & 78.04 & 75.02 & 72.10 & 71.43 & 69.30 & 66.18 & 63.18 & 79.06 & 76.26 & 73.90 & 72.36 \\
    GCN-DropNode & 83.04 & 78.16 & 75.82 & 72.74 & \textbf{72.12} & \textbf{69.84} & \textbf{66.96} & \textbf{63.48} & 79.00 & 76.64 & 74.06 & 72.48 \\
    GCN-DropMessage & \textbf{83.33} & \textbf{79.34} & \textbf{76.24} & \textbf{73.28} & 71.83 & 69.52 & 66.40 & 63.38 & \textbf{79.20} & \textbf{76.74} & \textbf{74.24} & \textbf{72.60} \\
    \bottomrule
    \end{tabular}
    }
\end{table*}

\vpara{Adversarial Attacks.}
We also conduct experiments to evaluate the effectiveness of random dropping methods against adversarial attacks.
We apply PGD attacks~\cite{xu2019topology} to perturb the graph structures on Cora and CiteSeer, using GCN as the backbone model.
Figure \ref{fig:attacks} presents the results.

\begin{figure}[h!]
\centering
\subfloat[Cora]{
\includegraphics[width=0.40\textwidth]{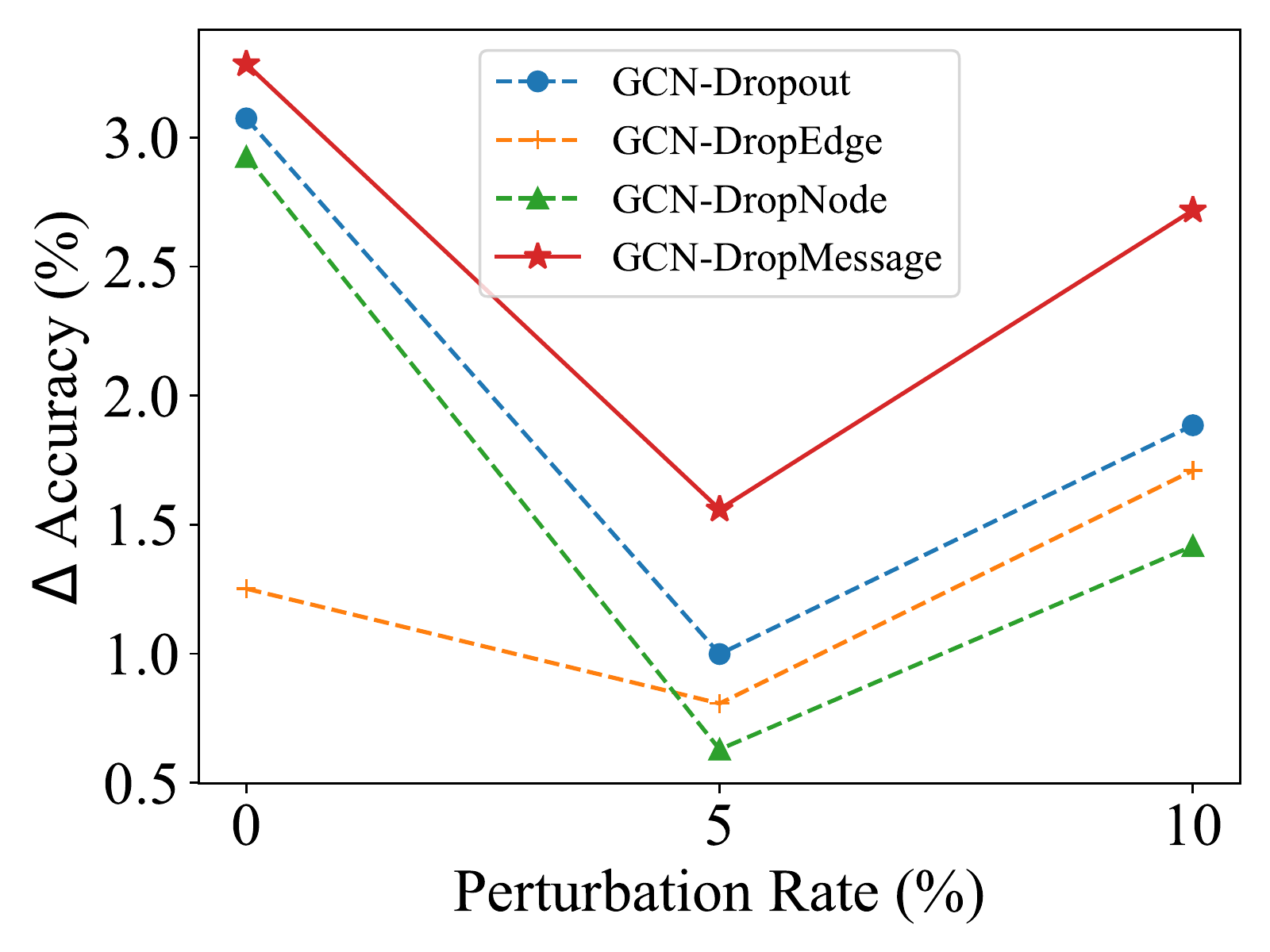}
}
\hspace{0.01\textwidth}
\subfloat[CiteSeer]{
\includegraphics[width=0.40\textwidth]{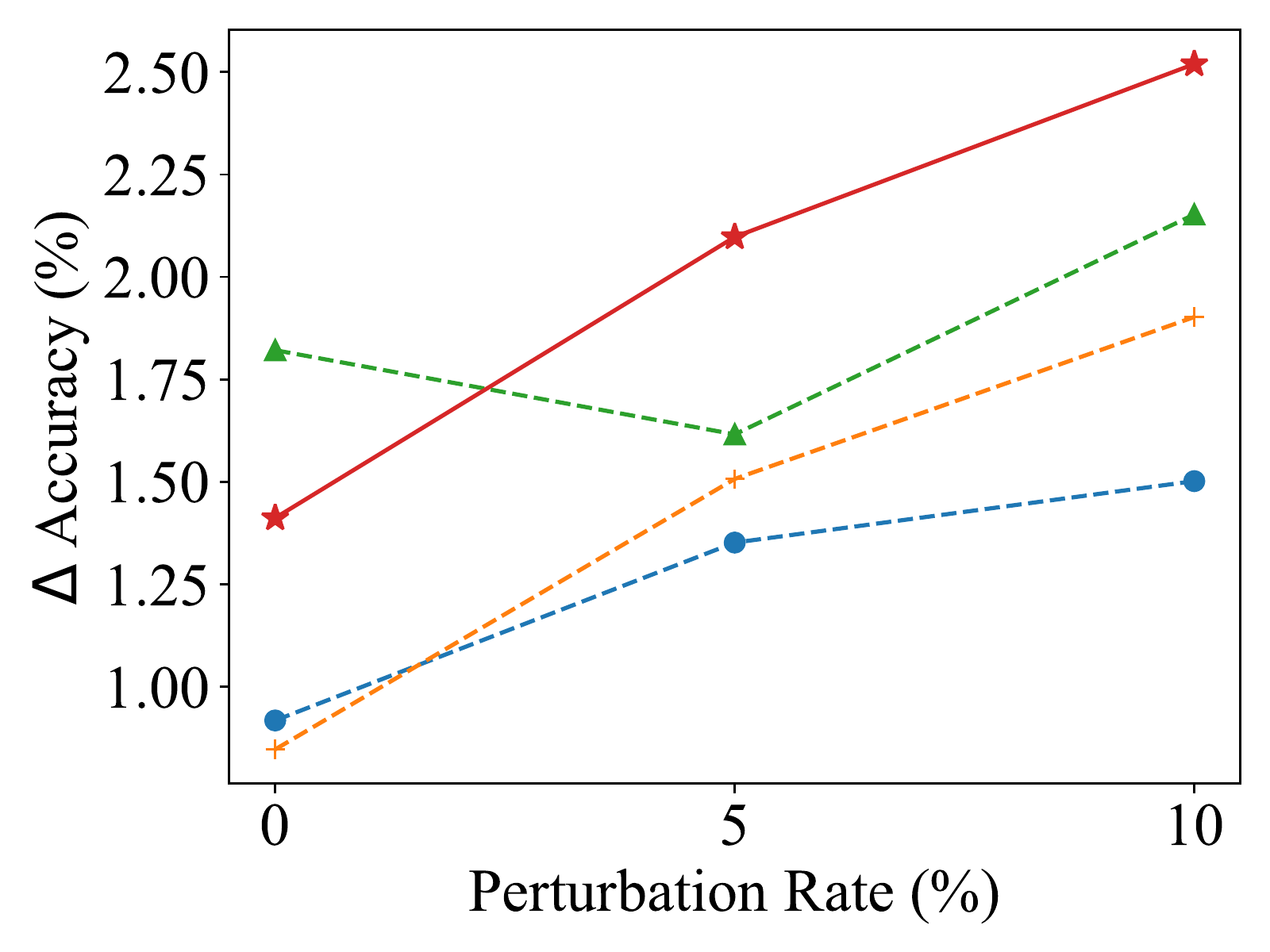}
}
\caption{
Model Performances Against PGD Attacks.
}
\label{fig:attacks}
\end{figure}

\subsection{Theoretical Analysis Towards Over-smoothing}

The over-smoothing issue extensively exists in graph neural networks~\cite{Chen2020MeasuringAR,Elinas2022AddressingOI}.
As the number of model layers increases, node representations become nearly indistinguishable, which leads to a significant decrease on model performance.
However, random dropping methods can alleviate this problem, and our proposed DropMessage achieves the best effect compared with other random dropping methods.

Now, we give a theoretical demonstration from the perspective of the information theory by measuring the Shannon entropy~\cite{Shannon2001AMT} of propagated messages.
The Shannon entropy quantifies ``the amount of information'', and measures the degree of confusion and diversity of the given variable, which can be expressed as:

\begin{align}
\mathrm{H}(X):=\mathbb{E}[-\log p(X)]=-\sum_{x \in \mathcal{X}} p(x) \log p(x)\nonumber
\end{align}
where $\mathcal{X}$ denotes all possible values of $x$.
Now, we measure the degree of over-smoothing of the model by calculating the Shannon entropy of the propagated messages.
Intuitively, the larger the Shannon entropy, the more diverse the propagated messages are, and the less likely the aggregated representations will converge.

We assume there are $k$ types of $d$-dimension messages propagated in the GNN model.
For the $i$-th type of the propagated messages, they are delivered for $t_i$ times by $n_i$ nodes.
Then, we can calculate the Shannon entropy for the initial messages:
\begin{align}
{\rm H}(clean)=\sum^{k}_{i}-p_{i}log \, p_i \nonumber    
\end{align}

Then, we consider the Shannon entropy after random dropping with dropping rate $\delta$.
For DropEdge and DropNode, they generate blank messages at the ratio $\delta$.
So the Shannon entropy can be expressed as:

\begin{align}
E({\rm H}(DropEdge)),E({\rm H}(DropNode))=-\delta log(\delta)+(1-\delta)\sum^{k}_{i}-p_{i}log((1-\delta)p_i)\nonumber
\end{align}

When it comes to Dropout, the Shannon entropy can be calculated as:

\begin{align}
E({\rm H}(Dropout))=\delta  \sum^{k}_{i} -p_i log(p_i/min\{d,n_i\})+(1-\delta)\sum^{k}_{i}-p_{i}log((1-\delta)p_i) \nonumber
\end{align}

As for our proposed DropMessage, the Shannon entropy can be expressed as:

\begin{align}
E({\rm H}(DropMessage))=\delta  \sum^{k}_{i} -p_i log(p_i/min\{d,t_i\})+(1-\delta)\sum^{k}_{i}-p_{i}log((1-\delta)p_i) \nonumber
\end{align}

Besides, we have $t_i\geq n_i$.
So, with proper dropping rate $\delta$, we can obtain:

\begin{align}
E({\rm H}(DropMessage))\geq E({\rm H}(Dropout)),E({\rm H}(DropEdge)),E({\rm H}(DropNode))\geq {\rm H}(clean) \nonumber
\end{align}

From above derivations, we prove the effectiveness of random dropping methods in alleviating over-smoothing issue, and our proposed DropMessage achieves the best effect.}

\end{document}